\documentclass[sigconf]{acmart}
\AtBeginDocument{%
  }

\renewcommand\footnotetextcopyrightpermission[1]{} 
\setcopyright{none}

\usepackage{microtype}
\usepackage{graphicx}
\usepackage{subfigure}
\usepackage{booktabs} 
\usepackage{multirow}  
\usepackage{array}  
\usepackage{xcolor} 
\usepackage{tcolorbox}
\usepackage{thmtools}
\usepackage{thmtools}
\usepackage{amsmath}
\usepackage{mathtools}
\usepackage{amsthm}
\usepackage{float}
\usepackage{makecell}
\begin{document}

\title{HyLiFormer: Hyperbolic Linear Attention for Skeleton-based Human Action Recognition}

\author{Yue Li}
\affiliation{%
  \institution{Sun Yat-sen University}
  \city{Shenzhen}
  \country{China}
}
\email{liyue228@mail2.sysu.edu.cn}

\author{Haoxuan Qu}
\affiliation{%
  \institution{Lancaster University}
  \city{Lancaster}
  \country{United Kingdom}
}
\email{h.qu5@lancaster.ac.uk}

\author{Mengyuan Liu}
\affiliation{%
  \institution{Peking University}
  \city{Shenzhen}
  \country{China}
}
\email{liumengyuan@pku.edu.cn}

\author{Jun Liu}
\affiliation{%
  \institution{Lancaster University}
  \city{Lancaster}
  \country{United Kingdom}
}
\email{j.liu81@lancaster.ac.uk}

\author{Yujun Cai}
\affiliation{%
  \institution{University of Queensland}
  \city{Brisbane}
  \country{Australia}
}
\email{vanora.caiyj@gmail.com}


\begin{abstract}
Transformers have demonstrated remarkable performance in skeleton-based human action recognition, yet their quadratic computational complexity remains a bottleneck for real-world applications. To mitigate this, linear attention mechanisms have been explored but struggle to capture the hierarchical structure of skeleton data. Meanwhile, the Poincaré model, as a typical hyperbolic geometry, offers a powerful framework for modeling hierarchical structures but lacks well-defined operations for existing mainstream linear attention. In this paper, we propose HyLiFormer, a novel hyperbolic linear attention Transformer tailored for skeleton-based action recognition. Our approach incorporates a Hyperbolic Transformation with Curvatures (HTC) module to map skeleton data into hyperbolic space and a Hyperbolic Linear Attention (HLA) module for efficient long-range dependency modeling. Theoretical analysis and extensive experiments on NTU RGB+D and NTU RGB+D 120 datasets demonstrate that HyLiFormer significantly reduces computational complexity while preserving model accuracy, making it a promising solution for efficiency-critical applications.
\end{abstract}





\maketitle

\section{Introduction}
Skeleton-based human action recognition (HAR) is a fundamental task in computer vision that aims to classify human actions from a sequence of time-continuous skeleton points. This task has found widespread applications in various domains  \citep{ren2024survey}, such as sports analytics, video surveillance, and human-computer interaction. 

Early HAR methods primarily relied on Convolutional Neural Networks (CNNs) to extract spatial features \citep{wang2018action}, while Recurrent Neural Networks (RNNs) were introduced to model temporal dynamics \citep{du2015hierarchical}.
\begin{figure}[h]
    \centering
	\begin{minipage}{0.65\linewidth}
        \centering
		\vspace{3pt}
		\centerline{\includegraphics[width=0.9\textwidth]{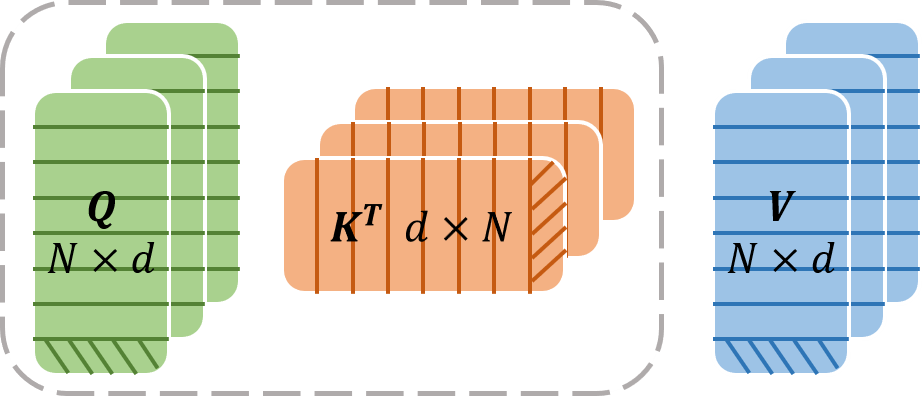}}
        \centerline{(a) Softmax Attention $\mathcal{O}(N^2d)$}
			\vspace{3pt}
        \centering
		\centerline{\includegraphics[width=0.9\textwidth]{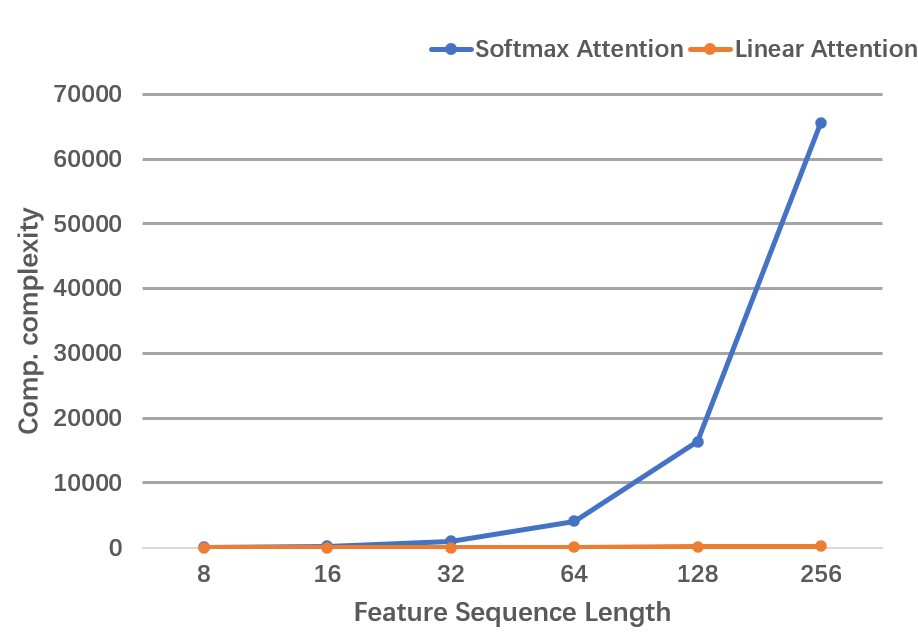}}
		\centerline{(b) Computational Complexity}
	\end{minipage}
	\caption{(a) The Process of Softmax Attention. The final attention matrix is computed by first multiplying  $Q$  and $K^T$, and then multiplying the result with $V$. Each row in the $QKV$ matrix (denoted by the slash in the figure) represents a temporal or spatial feature. It is evident that the computational complexity of Softmax Attention is $\mathcal{O}(N^2)$. (b) The Curve of Computational Complexity Growth with Feature Sequence Length. As the sequence length increases, Softmax Attention exhibits quadratic growth ($\mathcal{O}(N^2)$), whereas Linear Attention achieves significantly lower computational overhead with linear growth ($\mathcal{O}(N)$).}
    \label{fig:intro}
\end{figure}
Later, Graph Convolutional Networks (GCNs) achieved notable improvements by leveraging skeleton graph structures  \citep{yan2018spatial,chen2021channel,cheng2020skeleton}. However, most GCN-based methods typically focus on the feature information of individual nodes and their neighboring temporal/spatial nodes, assuming a fixed graph topology. Skeleton data inherently contains long-range temporal dependencies due to the sequential nature of human motion, while the relationships between joints are dynamic and flexible. Therefore, approaches that only consider neighboring temporal/spatial nodes are insufficient, and their capacity to capture long-range temporal dependencies and hierarchical relationships in skeleton data is inherently limited \citep{plizzari2021spatial}.

In recent years, Transformers \citep{vaswani2017attention} which was initially developed for Natural Language Processing (NLP) tasks, have achieved state-of-the-art performance across a range of domains \citep{huang2023hyperbolic,zhao2021point,parmar2018image}. Transformers are particularly adept at capturing long-range dependencies, and their application to skeleton data enables flexible modeling of both the spatial structure of the skeleton and long-term temporal dependencies. Likewise, they have gained substantial traction in the field of skeleton-based human action recognition \citep{do2025skateformer,ahn2023star,bai2022hierarchical,qiu2022spatio,zhang2021stst,plizzari2021skeleton}. 

However, a major drawback of transformer-based models is their computational and memory complexity, which grows quadratically with respect to the input sequence length \citep{peng2024eagle}. Specifically, as shown in Fig.\ref{fig:intro}(a), the computational complexity of traditional attention mechanisms is $\mathcal{O}(N^2)$, because given an input sequence of length $N$, the attention mechanism computes a pairwise similarity score for every pair of tokens in the sequence. This requires the generation of an $N \times N$ attention matrix, where each element represents the interaction between a query and a key. Consequently, both the computation of these interactions scale quadratically with the sequence length, leading to a complexity of $\mathcal{O}(N^2)$.

To mitigate this issue, linear attention mechanisms have been introduced \citep{peng2024eagle,han2023flatten,gu2023mamba} , reducing computational complexity to $\mathcal{O}(N)$ by approximating self-attention. As shown in Fig.\ref{fig:intro}(b), compared with traditional attention, linear attention can reduce the complexity from quadratic to linear with respect to sequence length, emerging as a promising solution for enabling efficient processing of long sequences in skeleton-based human action recognition. While this enhances efficiency, directly applying existing linear attention methods to HAR neglects the intrinsic hierarchical structure of skeleton data. Unlike textual sequences, which are inherently linear in structure, skeleton data is represented as tree-like graphs, where the hierarchical relationships between joints must be carefully preserved. Applying a standard linear attention mechanism, designed for flat, sequence-like data, to such structured inputs can lead to suboptimal performance due to its inability to account for the underlying tree topology. Moreover, standard linear attention mechanisms are often autoregressive \citep{gu2023mamba}, meaning they process inputs sequentially, which reduces their effectiveness in capturing global and bidirectional dependencies in hierarchical skeleton graph structures. This mismatch underscores the need for mechanisms capable of adapting linear attention to effectively handle non-Euclidean geometric structures, such as those represented by skeleton graphs. 

To address this challenge, we turn to hyperbolic geometry, which has demonstrated superior capabilities in modeling hierarchical data. Unlike Euclidean space, hyperbolic space provides exponential volume growth, enabling efficient representation of tree-like structures, making it an ideal candidate for skeleton-based human action recognition tasks \citep{ganea2018hyperbolic}. However, on the one hand, incorporating linear attention mechanisms in hyperbolic space adds complexity due to the non-Euclidean nature of the space \citep{qu2024llms}. On the other hand, many core operations used in standard linear attention, such as dot products and normalization, are not well-defined or computationally efficient in hyperbolic geometry. This incompatibility necessitates the development of new methods and adaptations to ensure that linear attention mechanisms can operate effectively within the hyperbolic domain.

Based on the above argument, in this paper, by leveraging the advantages of hyperbolic geometry while overcoming these operational challenges, we aim to pioneer a hyperbolic linear attention mechanism tailored for skeleton-based human action recognition, setting the stage for efficient and accurate modeling of hierarchical skeleton data. Specifically, we first perform an efficient and accurate data conversion of the skeleton data from Euclidean space to the Poincaré model in hyperbolic space through the HTC module, then model the linear attention mechanism in the HLA module, and finally convert the skeleton data from hyperbolic space back to Euclidean space through the inverse process of the HTC module. Overall, our contributions are summarized as follows:
\begin{enumerate}
    \item We propose HyLiFormer, a simple yet efficient hyperbolic linear attention transformer, which is the first linear attention mechanism designed specifically for the Poincaré model in hyperbolic space. This novel approach bridges the gap between the efficiency of linear attention and the need to model hierarchical data in hyperbolic geometry, facilitating effective skeleton-based action recognition.
    \item By incorporating the Hyperbolic Linear Attention (HLA) module, we achieve a significant reduction in computational complexity. The traditional quadratic complexity of self-attention is reduced to linear complexity, enabling the model to handle longer input sequences efficiently with minimal performance compromise.
    \item Our approach effectively preserves model performance even with the reduced computational cost. This enables the deployment of transformer-based models for skeleton-based action recognition in real-world applications, where both accuracy and computational efficiency are crucial. 
\end{enumerate}

\section{Related Work}
\subsection{Skeleton-based Action Recognition}
Skeleton-based human action recognition (HAR) has been extensively studied, evolving through multiple deep learning paradigms \citep{du2015skeleton,wang2018action,li2017joint,li2017adaptive,li2017skeleton}. Early methods for skeleton-based human action recognition relied primarily on convolutional neural networks (CNNs) to capture basic spatial interactions among skeleton points  \citep{du2015skeleton}. With the advent of Recurrent Neural Networks (RNNs) and Long Short-Term Memory (LSTM),  \citep{li2017adaptive,li2017skeleton} leveraged them to model temporal interactions. To better account for the topological structures of skeleton data, Graph Convolutional Networks (GCNs) have been extensively applied in this domain, achieving significant performance improvements  \citep{yan2018spatial,cheng2020skeleton,chen2021channel,zhou2024blockgcn}.  \citep{yan2018spatial} introduced ST-GCN, a spatiotemporal graph model that connects skeleton joints based on the natural body structure and temporal continuity.   \citep{cheng2020skeleton} proposed a network that contains spatial and temporal shift graph convolution.  \citep{chen2021channel} proposed a channel-wise topology graph convolution network (CTR-GCN) to dynamically capture spatial features at different levels of granularity.  \citep{zhou2024blockgcn} developed BlockGCN, a network designed to enhance the learning and retention of critical skeleton attributes. Collectively, these contributions represent significant advancements in utilizing the inherent graph structure of skeleton data for action recognition tasks. 

In contrast to the aforementioned approaches, our work introduces a novel hyperbolic-space-based linear attention mechanism. Benefiting from the linear attention design, our method achieves superior modeling of temporal dependencies compared to graph convolutional networks (GCNs), while leveraging the hyperbolic space to capture the hierarchical structure of skeletal data more effectively. Additionally, our approach is significantly more lightweight than transformer-based methods, addressing the challenges of high memory consumption without compromising much performance. To the best of our knowledge, this is the first application of linear attention mechanisms within the Poincaré model in hyperbolic space, addressing the dual challenges of high memory usage and hierarchical information modeling limitations in existing methods. 

\subsection{Hyperbolic Transformer}
In recent years, hyperbolic geometry has demonstrated significant potential for modeling complex structured data, particularly those with tree-like or hierarchical structures \citep{yang2024hypformer}. Numerous studies have begun exploring the application of transformers in hyperbolic space. For instance,  \citep{huang2023hyperbolic} employed a hyperbolic transformer for music generation, while  \citep{chen2024hyperbolic} utilized it for pre-trained language models. Additionally, hyperbolic geometry has been applied to model hierarchical skeleton data in skeleton-based human action recognition.  \citep{ermolov2022hyperbolic} introduced a hyperbolic vision transformer model featuring a novel metric learning loss that combines the representational power of hyperbolic space with the simplicity of cross-entropy loss.  \citep{chen2022hmanet} leveraged hyperbolic space mapping to enhance spatiotemporal feature representation, and  \citep{qu2024llms} integrated large language models with hyperbolic space to improve feature representation. In contrast to these approaches, our work is the first to explore linear attention mechanisms in hyperbolic spaces for skeleton-based human action recognition.

\section{Preliminary}
\label{sec:preliminary}
In this section, we will introduce some basics about Poincaré model in hyperbolic space and the two dominant Euclidean linear attention mechanisms briefly.
\subsection{Poincaré Model}
In this study, we adopt the Poincaré model as the hyperbolic space. To better understand the transformation formulas in the HTC module (introduced in Section \ref{sec:method}), it is essential to first define the Poincaré model and its role in mapping data from Euclidean space to hyperbolic space. An $n$-dimensional Poincaré model, denoted as $\mathbb{B}^{n}_{\kappa}$ is a Riemannian manifold $(\mathbb{B}^{n}_{\kappa},g^{\mathbb{B}}_{\mathbf{x}})$ with constant negative curvature $\kappa<0$. The Poincaré model is defined as
\begin{equation}
    \mathbb{B}^{n}_{\kappa}=\left\{\mathbf{x}\in\mathbb{R}^n:||\mathbf{x}||<-\frac{1}{\kappa}\right\}
\end{equation}
where $||\cdot||$ represents the Euclidean norm. Furthermore, the manifold is equipped with the Riemannian metric tensor
\begin{equation}
    g^{\mathbb{B}}_{\mathbf{x}}=\left(\frac{2}{1+\kappa||x||^2}\right)^2 g^{\mathbb{E}}
\end{equation}
where $x\in\mathbb{B}^{n}_{\kappa}$ and $g^{\mathbb{E}}$ denotes the Euclidean metric tensor. This formula demonstrates that hyperbolic geometry is a powerful framework for modeling hierarchical and structured data, where distances grow exponentially just like a tree structure. This property makes it particularly well-suited for skeleton-based human action recognition (HAR), as human motion inherently follows a multi-scale hierarchy.

\subsection{Receptance Weighted Key Value (RWKV)} 
RWKV \citep{peng2024eagle} is a relatively hot linear attention solution in recent years, which combines the advantages of both RNN and transformer and can realize parallelization of training, and in prediction can be realized with the linear growth of the predicted feature sequence, the prediction memory is also linear growth, greatly reducing the computational overhead. First, it performs a linear interpolation between the input data at the current time step $x_t$ and the previous time step $x_{t-1}$ to compute the matrices \ $r_t$,  $k_t$,  $v_t$, and  $g_t$
\begin{equation}
    \Box_t=W_{\Box}(\mu_{\Box}x_t+(1-\mu_{\Box})x_{t-1}), \quad \Box \in \{r,k,v,g\}
    \label{eq:rwkv1}
\end{equation}
Then, low-rank adaptation is applied to obtain $d_t$, which is subsequently exponentiated to yield $w_t$. After these steps, the model proceeds with the $wkv_t$ module, which is the core component of RWKV6. The equations for the $wkv_t$ module are defined as follows
\begin{equation}
    wkv_{t} =  \mathrm{d}(u)\cdot k_{t}^\mathrm{T} \cdot v_{t} + \sum_{i=1}^{t-1}  \mathrm{d}\left(\bigodot_{j=i+1}^{t-1}w_{j}\right) \cdot  k_{i}^\mathrm{T} \cdot v_{i}
    \label{eq:rwkv2}
\end{equation}
where $\mathrm{d}$ represents a diagonal matrix, and $\bigodot_{j=i+1}^{t-1} w_j$ denotes the Hadamard product. RWKV combines RNN and transformer advantages for efficient, scalable linear attention, reducing computational overhead.

\subsection{SSM-based Mamba}
The models based on Structured State Space (SSM), namely the Structured State Space Sequence model (S4) and Mamba \citep{gu2023mamba}, are inspired by continuous systems. The system operates by transforming the input \( x(t) \) to the output \( y(t) \) through a hidden state \( h(t) \), where \( h(t) \in \mathbb{R}^\mathtt{N} \). The evolution of the hidden state is governed by the parameter matrix \( \mathbf{A} \in \mathbb{R}^{\mathtt{N \times N}} \), while the input and output projections are defined by the matrices \( \mathbf{B} \in \mathbb{R}^{\mathtt{N} \times 1} \) and \( \mathbf{C} \in \mathbb{R}^{1 \times \mathtt{N}} \), respectively. The discrete versions of the system, namely S4 and Mamba, a time-step parameter \( \mathbf{\triangle} \) is introduced to convert the continuous parameters \( \mathbf{A} \) and \( \mathbf{B} \) into their discrete counterparts \( \mathbf{\overline{A}} \) and \( \mathbf{\overline{B}} \). A common approach for this conversion is Zero-Order Hold (ZOH), which is defined as follows:
\begin{equation}
    \begin{aligned}
        \mathbf{\overline{A}} &= \exp{(\mathbf{\triangle} \mathbf{A})}, \\
        \mathbf{\overline{B}} &= (\mathbf{\triangle} \mathbf{A})^{-1} (\exp{(\mathbf{\triangle} \mathbf{A})} - \mathbf{I}) \cdot \mathbf{\triangle} \mathbf{B}.
    \end{aligned}
    \label{eq:mamba1}
\end{equation}
These equations provide the discretization of the continuous system and define the discrete parameters used in the S4 and Mamba models. After discretizing $ \mathbf{\overline{A}} $ and $ \mathbf{\overline{B}} $, the system can be rewritten in its discrete form using a step size 
\begin{equation}
    \begin{aligned}
        h_t &= \mathbf{\overline{A}} h_{t-1} + \mathbf{\overline{B}} x_t, \\
        y_t &= \mathbf{C} h_t.
    \end{aligned}
    \label{eq:mamba2}
\end{equation}
Finally, the model computes the output via global convolution:
\begin{equation}
    \begin{aligned}
        \mathbf{\overline{K}} &= (\mathbf{C}\mathbf{\overline{B}}, \mathbf{C}\mathbf{\overline{A}}\mathbf{\overline{B}}, \dots, \mathbf{C}\mathbf{\overline{A}}^{\mathtt{M}-1} \mathbf{\overline{B}}), \\
        \mathbf{y} &= \mathbf{x} * \mathbf{\overline{K}},
    \end{aligned}
    \label{eq:mamba3}
\end{equation}
where $ \mathtt{M} $ is the length of the input feature sequence $ \mathbf{x} $, and $ \overline{\mathbf{K}} \in \mathbb{R}^{\mathtt{M}} $ represents a structured convolution kernel.

\textbf{Limitations.}\thinspace (1) \textbf{Unsuitable for hierarchical skeleton data.}\label{limitations1} These models were originally designed for natural language processing tasks and typically exhibit autoregressive properties, which make them more suited for modeling unidirectional linear data. For example, in Eq.\ref{eq:rwkv1} and Eq.\ref{eq:mamba2}, they generate each output by relying on previous outputs. However, skeleton data is hierarchical and bidirectional in temporal dimension, and the above linear attention mechanism limits their ability to effectively model the inherent hierarchical structure and bidirectional temporal information in the skeleton data. (2) \textbf{Poor definitions for operations in the hyperbolic Transformer.} \label{limitations2} Although the linear attention mechanisms discussed above can be directly applied to hyperbolic space to address the hierarchical modeling of skeleton data, existing linear attention mechanisms still face challenges when applied directly to hyperbolic space. As seen with RWKV and Mamba, certain key operations are not explicitly defined in hyperbolic space, such as matrix diagonalization in Eq.\ref{eq:rwkv1}, and matrix inversion in Eq.\ref{eq:mamba2} and Eq.\ref{eq:mamba3}. Alternatively,  \citep{zhang2021hype} defines an operation similar to the RWKV GRU, but it essentially maps the data back and forth between hyperbolic space and Euclidean space to avoid performing complex Euclidean operations in hyperbolic space. This either makes these models unsuitable for hyperbolic space or disrupts the continuity of the computation, resulting in significant computational overhead.
\section{Method}
\label{sec:method}
\begin{figure*}[t]
    \centering
    \includegraphics[width=0.75\linewidth]{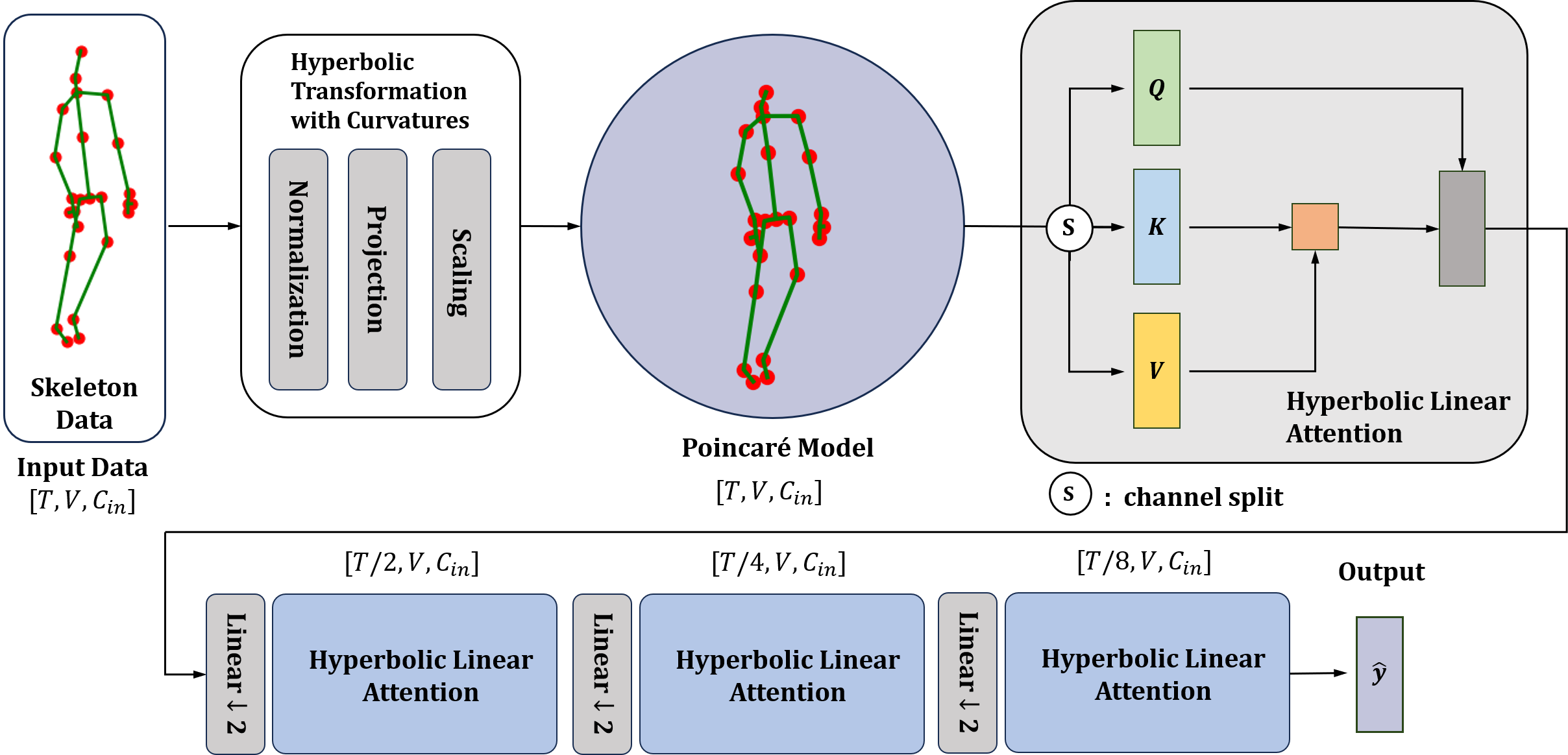}
    \caption{Framework of HyLiFormer. The input data (skeleton data) is projected onto the Poincaré model through Hyperbolic Transformation with Curvatures (HTC). The transformed data then passes through the hyperbolic linear attention block, which captures the temporal and hierarchical information of the skeleton data. Finally, the data is mapped back to Euclidean space using the inverse of the HTC, which is omitted in the diagram for simplicity.}
    \label{fig:model}
\end{figure*}
Skeleton-based human action recognition focuses on predicting human actions from a sequence of given skeleton points. The core challenge of this task lies in accurately and efficiently modeling the spatiotemporal relationships between skeleton joints. To address this, many existing methods leverage the self-attention mechanism in transformers  \citep{vaswani2017attention} to model the temporal and spatial dependencies of each joint relative to all others, achieving state-of-the-art performance  \citep{qiu2022spatio,zhang2021stst}.

However, the self-attention mechanism has a critical limitation: its computational complexity grows quadratically with the length of the input feature sequence. This significantly increases both training and inference time, hindering the deployment of transformer-based approaches in real-world applications where computational efficiency is crucial. To mitigate this limitation, we initially explored the use of mainstream linear attention mechanisms for skeletal data. Unfortunately, as discussed in Section~\ref{limitations1}, these methods face specific challenges when applied to this domain. We then attempted to model linear attention mechanisms in hyperbolic space but encountered additional issues, as detailed in Section~\ref{limitations2}.

Considering these challenges and the limitations of conventional self-attention mechanisms, inspired by  \citep{han2023flatten, ganea2018hyperbolic}, we propose HyLiFormer, a novel framework designed to efficiently learn and model the spatiotemporal information in skeleton data. HyLiFormer incorporates two key components: the Hyperbolic Transformation with Curvature (HTC) module, which projects the skeleton data from Euclidean space into the Poincaré model of hyperbolic space under curvature constraints, and the Hyperbolic Linear Attention (HLA) module, which performs self-attention operations in hyperbolic space. Finally, the HTC module maps the processed data back to Euclidean space. Below, we describe the details of the HTC module.

\subsection{Hyperbolic Transformation with Curvatures (HTC)}
Inspired by  \citep{balazevic2019multi}, the HTC module is incorporated to leverage hyperbolic geometry to enhance the linear attention mechanism's ability to model spatial relationships in skeleton data. This is achieved by performing an Euclidean-to-hyperbolic projection that embeds the skeleton data into hyperbolic space. Inspired by \citep{ganea2018hyperbolic}, the transformation involves computing unit vectors to preserve the directional information of the skeleton data, modeling the hierarchy of skeleton joints using the $\tanh$ function, and scaling the data to satisfy the constraints of the Poincaré model. 

Hyperbolic Möbius scalar multiplication \citep{ganea2018hyperbolic} has been widely utilized for hyperbolic embeddings. Building on this, we designed the HTC module to transform skeleton data from Euclidean space into the Poincaré model while preserving its hierarchical and geometric structure.

Formally, given a sequence of input skeleton point data $\mathbf{x} \in \mathbb{R}^{T \times V \times M \times C_{\text{in}}}$ and a target hyperbolic space $\mathbb{B}^{n}_{\kappa}$, where $T$ is the temporal length, $V$ is the number of joints per frame, $M$ is the number of individuals performing the action, and $\rm C_{in}$ is the dimensionality of the data for each joint, the transformation process consists of three steps, each described below.

\textbf{Unit Vector Calculation.}\thinspace To preserve the directional information of skeleton motion and capture the dynamics of joint movement, we compute the unit vectors $\mathbf{\hat{x}}$ of the input skeleton data $\mathbf{x}$. This step is defined as:
\begin{equation} 
    \mathbf{\hat{x}} = \frac{\mathbf{x}}{||\mathbf{x}||}, \quad ||\mathbf{x}|| = \sqrt{\sum_{i=1}^{\rm C_{in}} x_i^2}    
\end{equation} 
where $||\mathbf{x}||$ is the Euclidean norm of each data point, and $\mathbf{\hat{x}}$ retains the directional information while normalizing the magnitude. This operation effectively separates the motion direction from the magnitude, allowing subsequent steps to focus on the hierarchical structure.

\textbf{Hierarchy Modeling.}\thinspace To encode the hierarchical structure of skeleton data, we apply the hyperbolic tangent ($\tanh$) function to the normalized magnitude $||\mathbf{x}||$. The transformation is defined as: \begin{equation} 
    \mathbf{\tilde{x}} = \tanh\left(-\kappa \cdot ||\mathbf{x}||\right), 
\end{equation} 
where $\kappa$ represents the curvature of the hyperbolic space. The $\tanh$ function compresses the range of $||\mathbf{x}||$ into $(-1, 1)$, enabling the mapping of hierarchical levels in hyperbolic space. Small magnitudes ($||\mathbf{x}|| \to 0$) are mapped closer to the origin of the Poincaré model, corresponding to global features (e.g., overall motion). Large magnitudes ($||\mathbf{x}|| \to \infty$) are mapped near the boundary, corresponding to local features (e.g., detailed joint-level actions). The output $\mathbf{\tilde{x}}$ represents the scaled magnitude while preserving the original directionality through $\mathbf{\hat{x}}$.

\textbf{Scaling for Hyperbolic Constraints.}\thinspace Finally, to ensure the transformed data adheres to the constraints of the Poincaré model (i.e., staying within the unit ball in hyperbolic space), we combine the normalized direction $\mathbf{\hat{x}}$ and scaled magnitude $\mathbf{\tilde{x}}$ to compute the final hyperbolic representation
\begin{equation} 
    \mathbf{x^\mathbb{B}_{\kappa}} = \mathbf{\tilde{x}} \cdot \mathbf{\hat{x}} = -\frac{1}{\kappa} \cdot \tanh\left(-\kappa \cdot ||\mathbf{x}||\right) \cdot \mathbf{\hat{x}}. 
\end{equation} 
Here, $\mathbf{x^\mathbb{B}_{\kappa}} \in \mathbb{B}^{n}_{\kappa}$ represents the transformed skeleton data in hyperbolic space. This transformation preserves the hierarchical structure, while ensuring compatibility with the curvature $\kappa$ of the target Poincaré model. It is rigorously proved as follows

\begin{lemma}\label{lem:htc_poincare_constraint}
Given an input skeleton data point $\mathbf{x} \in \mathbb{R}^{T \times V \times M \times C_{\text{in}}}$, the transformation applied by the HTC module ensures that the output $\mathbf{x^{\mathbb{B}}_{\kappa}}$ satisfies the Poincar\'e model constraint, i.e., $\|\mathbf{x^{\mathbb{B}}_{\kappa}}\| < -\frac{1}{\kappa}$. 
\end{lemma}

\begin{proof}\renewcommand{\qedsymbol}{}
The transformation consists of unit vector computation, hierarchy modeling, and hyperbolic scaling. First, the unit vector of the input is computed as $\mathbf{\hat{x}} = \frac{\mathbf{x}}{||\mathbf{x}||}$, ensuring $||\mathbf{\hat{x}}|| = 1$. Next, hierarchy modeling applies $\mathbf{\tilde{x}} = \tanh(-\kappa ||\mathbf{x}||)$, which maps magnitudes to $(-1,1)$. Finally, the transformed representation is given by $\mathbf{x^{\mathbb{B}}_{\kappa}} = -\frac{1}{\kappa} \tanh(-\kappa ||\mathbf{x}||) \cdot \mathbf{\hat{x}}$. Since $||\tanh(-\kappa ||\mathbf{x}||)|| < 1$, it follows that $\|\mathbf{x^{\mathbb{B}}_{\kappa}}\| < -\frac{1}{\kappa}$, ensuring the result remains within the Poincar\'e model. 
\end{proof}

Through this process, the skeleton data is efficiently and effectively transformed from Euclidean space to hyperbolic space within the Poincaré model, enabling better hierarchical representation and spatiotemporal modeling in subsequent stages of HyLiFormer. The operation of mapping data from hyperbolic space back to Euclidean space within the Poincaré model is simply the inverse of all operations in the HTC module, so we will not elaborate on it here.

\subsection{Hyperbolic Linear Attention (HLA)}
The Hyperbolic Linear Attention (HLA) module is proposed to address the computational bottlenecks in traditional Transformer-based attention mechanisms, specifically the quadratic complexity $\mathcal{O}(N^2)$ with respect to the sequence length $N$, which hinders scalability and efficiency when processing long sequences. To overcome this limitation, inspired by  \citep{han2023flatten} \citep{tsai2019transformer}, we design a hyperbolic linear attention and rigorously prove that it satisfies the hyperbolic space Poincaré model constraint.

\textbf{Softmax Attention.}
We first recall the general form of self-attention in Euclidean Transformers, represented by the following weighted sum of value vectors, with weights determined by the similarity between the query and key vectors: \begin{equation} 
\begin{aligned} 
Q &= x W_Q, K = x W_K, V = x W_V, \\ V_i'&=\sum_{j=1}^{N}\ \frac{{\rm Sim}{\left(Q_i,K_j\right)}}{\sum_{j=1}^{N}\ {\rm Sim}{\left(Q_i,K_j\right)}}V_j. 
\label{eq:attn} 
\end{aligned} 
\end{equation} 
where $W_Q, W_K, W_V \in \mathbb{R}^{F \times F}$ are projection matrices and $\rm Sim(\cdot, \cdot)$ denotes the similarity function. The computational complexity of the above formula is $\mathcal{O}(N^2)$ because the weighted summation requires traversing all key-value pairs, such as the calculation of the similarity matrix $QK^T$.

To overcome this limitation, we transform the traditional attention mechanism into a form that avoids the explicit calculation of the similarity matrix $QK^T$. Our hyperbolic linear attention mechanism reduces the computational cost to $\mathcal{O}(N)$ by using matrix multiplication properties.

Specifically, let $\mathbf{x} \in \mathbb{B}^{n}{\kappa}$ be the input sequence of skeleton data, with the shape $\mathbf{x} \in \mathbb{R}^{T \times V \times M \times C{\text{in}}}$, where the symbols have the same meaning as in the HTC module. The attention mechanism is computed using the following steps:

\textbf{Query, Key, and Value Matrices.} \thinspace 
We first define the query ($Q$), key ($K$), and value ($V$) matrices, which are derived from the input $\mathbf{x}$, consistent with the traditional self-attention mechanism. These are computed as: \begin{equation} Q = \mathbf{x}[1:], \quad K = \mathbf{x}[2:], \quad V = \mathbf{x}[3:],
\end{equation} where $Q$, $K$, and $V$ represent the query, key, and value matrices, respectively.

\textbf{Reformulating the Attention Mechanism.} \thinspace 
The output of the traditional attention mechanism is computed as a weighted sum of the value vectors, with weights determined by the similarity between the query and key vectors. To avoid the explicit calculation of the similarity matrix $QK^T$ in Eq. \ref{eq:attn} and achieve linear attention, we propose an approximation by reordering the operations. The reformulated attention mechanism is given by
\begin{equation}
    V_i=Q_i\cdot \mathrm{Sim}(K_j^T\cdot V_j)
\end{equation}
where $Q_i$ and $K_j$ represent feature-mapped representations of $Q$ and $K$, obtained through a kernel transformation $\phi(\cdot)$. The function $\mathrm{Sim}(\cdot)$ defines a similarity operation applied to the transformed key-value pairs.

In this formulation, $\mathrm{Sim}(K_j^T\cdot V_j)$ aggregates the key-value pairs into a fixed-size representation with dimensions determined by the feature space $\mathbb{R}^F$, independent of the sequence length $N$. This allows $Q_i$ to interact with a fixed-size representation, effectively reducing the computational complexity from $\mathcal{O}(N^2F)$ to $\mathcal{O}(NF^2)$. By reordering the operations, the reformulation avoids directly computing the $N \times N$ similarity matrix $QK^T$, significantly improving efficiency for long sequences.

The choice of similarity function $\mathrm{Sim}(\cdot)$ plays a critical role in maintaining the expressiveness of the attention mechanism while ensuring numerical stability and computational efficiency. In our work, we adopt the softmax function as the kernel, defined as $\phi(x) = \exp(x)$. This choice captures the relative importance of elements effectively and ensures numerical stability in hyperbolic space.  It is rigorously proved as follows.

\begin{lemma}\label{lem:hla_poincare_constraint}
Given an input sequence of skeleton data $\mathbf{x} \in \mathbb{B}^{n}_{\kappa}$, the transformation applied by the HLA module ensures that the output $\mathbf{V}$ satisfies the Poincar\'e ball constraint, i.e., $\|\mathbf{V}\| < -\frac{1}{\kappa}$.
\end{lemma}
\begin{proof}\renewcommand{\qedsymbol}{}
The HLA module reformulates the traditional attention mechanism while preserving the hyperbolic structure. Given $\mathbf{x} \in \mathbb{B}^{n}_{\kappa}$, the query, key, and value matrices are computed as $Q = \mathbf{x}[1:], K = \mathbf{x}[2:], V = \mathbf{x}[3:]$. The hyperbolic linear attention is defined as $V_i = Q_i \cdot \mathrm{Sim}(K_j^T \cdot V_j)$, where $\mathrm{Sim}(\cdot)$ is a similarity function based on a kernel transformation $\phi(\cdot)$. To ensure that $V_i$ satisfies the Poincar\'e model constraint, we analyze each component of this equation.

Since $Q_i$ is derived from $\mathbf{x}$, it inherits the norm bound $\|Q_i\| < -\frac{1}{\kappa}$. The similarity function $\mathrm{Sim}(K_j^T \cdot V_j)$ is designed such that it preserves hyperbolic distances and results in outputs bounded by $(-1,1)$. Therefore, we have:
\begin{equation}
\|V_i\| = \|Q_i \cdot \mathrm{Sim}(K_j^T \cdot V_j)\|= \|Q_i\| \cdot ||1||< -\frac{1}{\kappa}.
\end{equation}
Thus, the transformed representations remain within the Poincar\'e model, ensuring that the HLA module satisfies the hyperbolic space constraint.
\end{proof}

Through this procedure, the hyperbolic linear attention module avoids the direct computation of the $QK^T$ similarity matrix, reducing the computational complexity to $\mathcal{O}(NF^2)$. This not only enhances efficiency but also preserves the expressive power and numerical stability of the attention mechanism.
\section{Experiment}
\begin{table*}[t]
    \footnotesize
    \centering
    \caption{Comparison of recognition efficiency performances against transformer-based methods on the NTU-RGB+D 60 and NTU-RGB+D 120 datasets under the joint modality. Bold text indicates the optimal performance, while dagger marks (${ \_ }$) denote the second-best performance.}
    \vspace{0.15cm}
    \label{tab:training time}
    \resizebox{1\textwidth}{!}{
    \def\arraystretch{1.25} 
    \begin{tabular}{l|c|c|c|c|c|c|c}
        \hline
        \multirow{2}{*}{Methods} & \multicolumn{2}{c|}{NTU-RGB+D 60 (\%)} & \multicolumn{2}{c|}{NTU-RGB+D 120 (\%)} & \multirow{2}{*}{\makecell{Training Time \\ (min/epoch)}} & \multirow{2}{*}{\makecell{Params.\\(M)}} & \multirow{2}{*}{\makecell{FLOPs(G)}}\\ \cline{2-5}
        & X-Sub60 & X-View60 & X-Sub120 & X-Set120 & & & \\ 
        \hline
        ST-TR \citep{plizzari2021spatial} & 89.9 & 96.1 & 81.9 & 84.1 & - & 19.4 &  57.6\\ 
        STTFormer \citep{qiu2022spatio} & 89.9 & 94.3 & - & - & 8.8 & 6.4 & 41.7 \\
        Zoom Transformer \citep{zhang2022zoom} & 90.1 & 95.3 & 84.8 & 86.5 & - & 4.8 & 5.6 \\
        HyperFormer \citep{ding2023hyperformer} & 90.7 & 95.1 & 86.6 & 88.0 & 8.3 & 2.7 & 14.8 \\  
        FreqMixFormer \citep{wu2024frequency} & 91.5 & 96.0 & \textbf{87.9} & \underline{89.1} & 40.3 & 2.1 & \textbf{2.4}\\
        SkateFormer \citep{do2025skateformer} & \textbf{92.6} & \textbf{97.0} & \underline{87.7} & \textbf{89.3} & \underline{5.0} & \underline{2.0} & 3.6  \\
        \hline
        HyLiFormer & \underline{91.7} & \underline{96.2} & 87.5 & 88.6 & \textbf{3.7} & \textbf{1.9} & \underline{3.5} \\ \hline
    \end{tabular}}
\end{table*}
\subsection{Datasets}
To validate the effectiveness and generalizability of our proposed method, we conduct experiments on two widely-used datasets for skeleton-based human action recognition: NTU RGB+D and NTU RGB+D 120 Dataset. These datasets provide comprehensive benchmarks with multi-modality information, including depth maps, 3D skeleton joint positions, RGB frames, and infrared sequences. However, our study focuses solely on the skeleton joint unimodality.

\textbf{NTU RGB+D.}\thinspace This dataset includes 60 action categories with 56,880 samples collected from 40 participants. The actions are categorized into three main groups: 40 daily activities (e.g., drinking, eating, reading), 9 health-related actions (e.g., sneezing, staggering, falling), and 11 interactive actions (e.g., punching, kicking, hugging). Evaluation protocols include two settings: cross-subject testing (X-Sub60), where participants are divided into distinct training and testing groups, and cross-view testing (X-View60), where data from one camera is used for testing and the other two for training.

\textbf{NTU RGB+D 120.}\thinspace As an extended version of NTU RGB+D, this dataset contains 120 action categories, covering a broader range of daily activities, mutual interactions, and health-related actions. It includes over 114,000 video samples and more than 8 million frames, collected from 106 participants. The evaluation protocols are consistent with NTU RGB+D, employing cross-subject testing (X-Sub120) and cross-view testing (X-Set120) to ensure robust benchmarking.

\subsection{Experiment Details}
All experiments were run on a single NVIDIA RTX 3090 GPU, and the framework and optimizer followed  \citep{do2025skateformer}. We found that our model has different learning rates for datasets of varying complexity, so the learning rate varies across datasets. For the experiments on the NTU RGB+D and NTU RGB+D 120 datasets, we used the following configurations: $\rm V = 48$ (excluding the body centers, so both individuals have 24 joints each), $\rm T = 64, L = 4, K = 12, M = 8, N = 8,$ and $\rm C = 96$. Also, the curvature of the Poincaré model in hyperbolic space is chosen to be -1.

\subsection{Training Time Comparison}
We compare the training time per epoch and performance of HyLiFormer with recent state-of-the-art skeleton-based action recognition methods. The comparison focuses on the training time and accuracy of the skeleton unimodal data and the modular efficiency of the attention mechanism. Table \ref{tab:training time} shows the overall training time, module time, and accuracy comparison of the various methods for the NTU RGB+D and NTU RGB+D 120 datasets for the joint modality.

\subsection{Ablation Studies}
In order to further explore the validity of our HyLiFormer, the ablation experiments were carried on the X-Sub protocol of the NTU RGB+D 120 dataset. The detailed analysis is provided in Tables \ref{tab:2}, \ref{tab:3}, \ref{tab:4} to quantify the impact of different design choices on the model’s performance.

\textbf{Impact of the Curvature of Poincaré Model.}\thinspace In our framework, as inspired by  \citep{han2023flatten}, we model the linear attention mechanism within the Poincaré model in hyperbolic space. As described in the definition of the Poincaré model in Section \ref{sec:preliminary}, the curvature is a critical hyperparameter that influences the geometry and performance of the model. To explore the effect of different curvatures, we control for all other variables and conduct experiments with varying curvature values, as presented in Table \ref{tab:2}. Our results demonstrate that the framework achieves optimal performance when the curvature is set to -1, indicating that this specific curvature is most conducive to the model's ability to capture the structural properties of the data. Consequently, we adopt a curvature of -1 in our experiments to construct all Poincaré models, ensuring consistent and optimal performance throughout our evaluation.
\begin{table}[h]
\centering
\renewcommand\arraystretch{1.05}
\resizebox{0.4\textwidth}{!}{
\begin{tabular}{c|c}
\hline
Curvature of Poincare Model & Accuracy(\%) \\ \hline
$\kappa=-1$    & \textbf{87.5}        \\ \hline
$\kappa=-2$    & 87.0        \\ \hline
$\kappa=-3$    & 87.1        \\ \hline
\end{tabular}
}
\caption{Comparison of the performance of Poincaré models with different curvatures according to the X-Sub protocol on the NTU-RGB 120 dataset.}
\label{tab:2}
\end{table}

\textbf{Exploration of Existing Linear Attention Mechanisms Applied to Skeleton Data.}\thinspace As outlined in Section \ref{sec:preliminary} and Section \ref{sec:method}, we initially attempted to apply existing mainstream 
\begin{table}[h]
\centering
\renewcommand\arraystretch{1.05}
\begin{tabular}{c|c|c}
\hline
Method  & Accuracy(\%) & \makecell{Training Time\\(min/epoch)} \\ \hline
\begin{tabular}[c]{@{}c@{}}RWKV\\  \citep{peng2024eagle}\end{tabular}  & 86.8    &  5.5 \\ \hline
\begin{tabular}[c]{@{}c@{}}Mamba\\  \citep{gu2023mamba} \end{tabular} & 86.7    &  5.3  \\ \hline
Ours  & \textbf{87.5}     &   \textbf{3.7}  \\ \hline
\end{tabular}
\caption{Comparison of the effects of directly applying existing linear attention mechanisms to skeleton data.}
\label{tab:3}
\end{table}
linear attention mechanisms directly to skeleton data, with the expectation that they would offer a viable solution for skeleton-based human action recognition. However, our experiments revealed that the performance of these mechanisms was unsatisfactory. The results of our experiments are presented in Table \ref{tab:3}, which clearly illustrates the disparity in performance between these existing linear attention mechanisms and our proposed approach. In both terms of training time and final accuracy, the performance of the existing methods fell short.

\textbf{Exploration of Existing Attention Mechanisms in Hyperbolic Spaces} \thinspace As discussed in Section \ref{sec:preliminary} and Section \ref{sec:method}, we also investigated the application of existing mainstream linear attention mechanisms in hyperbolic space. 
\begin{table}[h]
\centering
\renewcommand\arraystretch{1.05}
\begin{tabular}{c|c|c}
\hline
Method & Accuracy(\%) & \makecell{Training Time \\ (min/epoch)} \\ \hline
\begin{tabular}[c]{@{}c@{}}Hyperbolic RWKV\\  \citep{peng2024eagle}\end{tabular}  & 87.2    &     18.2   \\ \hline
\begin{tabular}[c]{@{}c@{}}Hyperbolic Mamba\\  \citep{gu2023mamba} \end{tabular} & 87.0  &   22.4   \\ \hline
Ours  & \textbf{87.5} &    \textbf{3.7}       \\ \hline
\end{tabular}
\caption{Comparison of the effects of applying existing linear attention mechanisms to hyperbolic space.}
\label{tab:4}
\end{table}
However, we encountered challenges with the compatibility of certain Euclidean operations within the module when applied directly to hyperbolic space. Specifically, some Euclidean operations are not well-defined in the hyperbolic space. In this approach, we perform the undefined operations in Euclidean space and rely on the already defined operations in hyperbolic space. This results in a process of continually mapping the data back and forth between Euclidean space and hyperbolic space during computations. The computational results of this method are presented in Table \ref{tab:4}. Although we observed a slight improvement in the final performance, the trade-off was a significant increase in training time, highlighting the inefficiency of this approach.
\section{Conclusion}
This paper proposes HyLiFormer, a novel hyperbolic linear attention Transformer for skeleton-based human action recognition. By integrating hyperbolic geometry with linear attention, our model achieves efficient hierarchical and temporal modeling while reducing computational complexity from $\mathcal{O}(N^2)$ to $\mathcal{O}(N)$. Extensive experiments demonstrate that HyLiFormer maintains high recognition accuracy while significantly improving efficiency, making it well-suited for real-world applications.
\section*{Impact Statements}
This paper advances Machine Learning by improving efficiency in skeleton-based action recognition. Potential societal impacts include applications in surveillance, healthcare, and human-computer interaction. 

\bibliographystyle{ACM-Reference-Format}
\bibliography{references}


\begin{thebibliography}{37}


\ifx \showCODEN    \undefined \def \showCODEN     #1{\unskip}     \fi
\ifx \showISBNx    \undefined \def \showISBNx     #1{\unskip}     \fi
\ifx \showISBNxiii \undefined \def \showISBNxiii  #1{\unskip}     \fi
\ifx \showISSN     \undefined \def \showISSN      #1{\unskip}     \fi
\ifx \showLCCN     \undefined \def \showLCCN      #1{\unskip}     \fi
\ifx \shownote     \undefined \def \shownote      #1{#1}          \fi
\ifx \showarticletitle \undefined \def \showarticletitle #1{#1}   \fi
\ifx \showURL      \undefined \def \showURL       {\relax}        \fi
\providecommand\bibfield[2]{#2}
\providecommand\bibinfo[2]{#2}
\providecommand\natexlab[1]{#1}
\providecommand\showeprint[2][]{arXiv:#2}

\bibitem[Ahn et~al\mbox{.}(2023)]%
        {ahn2023star}
\bibfield{author}{\bibinfo{person}{Dasom Ahn}, \bibinfo{person}{Sangwon Kim}, \bibinfo{person}{Hyunsu Hong}, {and} \bibinfo{person}{Byoung~Chul Ko}.} \bibinfo{year}{2023}\natexlab{}.
\newblock \showarticletitle{Star-transformer: a spatio-temporal cross attention transformer for human action recognition}. In \bibinfo{booktitle}{\emph{Proceedings of the IEEE/CVF winter conference on applications of computer vision}}. \bibinfo{pages}{3330--3339}.
\newblock


\bibitem[Bai et~al\mbox{.}(2022)]%
        {bai2022hierarchical}
\bibfield{author}{\bibinfo{person}{Ruwen Bai}, \bibinfo{person}{Min Li}, \bibinfo{person}{Bo Meng}, \bibinfo{person}{Fengfa Li}, \bibinfo{person}{Miao Jiang}, \bibinfo{person}{Junxing Ren}, {and} \bibinfo{person}{Degang Sun}.} \bibinfo{year}{2022}\natexlab{}.
\newblock \showarticletitle{Hierarchical graph convolutional skeleton transformer for action recognition}. In \bibinfo{booktitle}{\emph{2022 IEEE International Conference on Multimedia and Expo (ICME)}}. IEEE, \bibinfo{pages}{01--06}.
\newblock


\bibitem[Balazevic et~al\mbox{.}(2019)]%
        {balazevic2019multi}
\bibfield{author}{\bibinfo{person}{Ivana Balazevic}, \bibinfo{person}{Carl Allen}, {and} \bibinfo{person}{Timothy Hospedales}.} \bibinfo{year}{2019}\natexlab{}.
\newblock \showarticletitle{Multi-relational poincar{\'e} graph embeddings}.
\newblock \bibinfo{journal}{\emph{Advances in Neural Information Processing Systems}}  \bibinfo{volume}{32} (\bibinfo{year}{2019}).
\newblock


\bibitem[Chen et~al\mbox{.}(2022)]%
        {chen2022hmanet}
\bibfield{author}{\bibinfo{person}{Jinghong Chen}, \bibinfo{person}{Chong Zhao}, \bibinfo{person}{Qicong Wang}, {and} \bibinfo{person}{Hongying Meng}.} \bibinfo{year}{2022}\natexlab{}.
\newblock \showarticletitle{Hmanet: Hyperbolic manifold aware network for skeleton-based action recognition}.
\newblock \bibinfo{journal}{\emph{IEEE Transactions on Cognitive and Developmental Systems}} \bibinfo{volume}{15}, \bibinfo{number}{2} (\bibinfo{year}{2022}), \bibinfo{pages}{602--614}.
\newblock


\bibitem[Chen et~al\mbox{.}(2024)]%
        {chen2024hyperbolic}
\bibfield{author}{\bibinfo{person}{Weize Chen}, \bibinfo{person}{Xu Han}, \bibinfo{person}{Yankai Lin}, \bibinfo{person}{Kaichen He}, \bibinfo{person}{Ruobing Xie}, \bibinfo{person}{Jie Zhou}, \bibinfo{person}{Zhiyuan Liu}, {and} \bibinfo{person}{Maosong Sun}.} \bibinfo{year}{2024}\natexlab{}.
\newblock \showarticletitle{Hyperbolic Pre-Trained Language Model}.
\newblock \bibinfo{journal}{\emph{IEEE/ACM Transactions on Audio, Speech, and Language Processing}} (\bibinfo{year}{2024}).
\newblock


\bibitem[Chen et~al\mbox{.}(2021)]%
        {chen2021channel}
\bibfield{author}{\bibinfo{person}{Yuxin Chen}, \bibinfo{person}{Ziqi Zhang}, \bibinfo{person}{Chunfeng Yuan}, \bibinfo{person}{Bing Li}, \bibinfo{person}{Ying Deng}, {and} \bibinfo{person}{Weiming Hu}.} \bibinfo{year}{2021}\natexlab{}.
\newblock \showarticletitle{Channel-wise topology refinement graph convolution for skeleton-based action recognition}. In \bibinfo{booktitle}{\emph{Proceedings of the IEEE/CVF international conference on computer vision}}. \bibinfo{pages}{13359--13368}.
\newblock


\bibitem[Cheng et~al\mbox{.}(2020)]%
        {cheng2020skeleton}
\bibfield{author}{\bibinfo{person}{Ke Cheng}, \bibinfo{person}{Yifan Zhang}, \bibinfo{person}{Xiangyu He}, \bibinfo{person}{Weihan Chen}, \bibinfo{person}{Jian Cheng}, {and} \bibinfo{person}{Hanqing Lu}.} \bibinfo{year}{2020}\natexlab{}.
\newblock \showarticletitle{Skeleton-based action recognition with shift graph convolutional network}. In \bibinfo{booktitle}{\emph{Proceedings of the IEEE/CVF conference on computer vision and pattern recognition}}. \bibinfo{pages}{183--192}.
\newblock


\bibitem[Ding et~al\mbox{.}(2023)]%
        {ding2023hyperformer}
\bibfield{author}{\bibinfo{person}{Kaize Ding}, \bibinfo{person}{Albert~Jiongqian Liang}, \bibinfo{person}{Bryan Perozzi}, \bibinfo{person}{Ting Chen}, \bibinfo{person}{Ruoxi Wang}, \bibinfo{person}{Lichan Hong}, \bibinfo{person}{Ed~H Chi}, \bibinfo{person}{Huan Liu}, {and} \bibinfo{person}{Derek~Zhiyuan Cheng}.} \bibinfo{year}{2023}\natexlab{}.
\newblock \showarticletitle{HyperFormer: Learning Expressive Sparse Feature Representations via Hypergraph Transformer}. In \bibinfo{booktitle}{\emph{Proceedings of the 46th International ACM SIGIR Conference on Research and Development in Information Retrieval}}. \bibinfo{pages}{2062--2066}.
\newblock


\bibitem[Do and Kim(2025)]%
        {do2025skateformer}
\bibfield{author}{\bibinfo{person}{Jeonghyeok Do} {and} \bibinfo{person}{Munchurl Kim}.} \bibinfo{year}{2025}\natexlab{}.
\newblock \showarticletitle{Skateformer: skeletal-temporal transformer for human action recognition}. In \bibinfo{booktitle}{\emph{European Conference on Computer Vision}}. Springer, \bibinfo{pages}{401--420}.
\newblock


\bibitem[Du et~al\mbox{.}(2015a)]%
        {du2015skeleton}
\bibfield{author}{\bibinfo{person}{Yong Du}, \bibinfo{person}{Yun Fu}, {and} \bibinfo{person}{Liang Wang}.} \bibinfo{year}{2015}\natexlab{a}.
\newblock \showarticletitle{Skeleton based action recognition with convolutional neural network}. In \bibinfo{booktitle}{\emph{2015 3rd IAPR Asian conference on pattern recognition (ACPR)}}. IEEE, \bibinfo{pages}{579--583}.
\newblock


\bibitem[Du et~al\mbox{.}(2015b)]%
        {du2015hierarchical}
\bibfield{author}{\bibinfo{person}{Yong Du}, \bibinfo{person}{Wei Wang}, {and} \bibinfo{person}{Liang Wang}.} \bibinfo{year}{2015}\natexlab{b}.
\newblock \showarticletitle{Hierarchical recurrent neural network for skeleton based action recognition}. In \bibinfo{booktitle}{\emph{Proceedings of the IEEE conference on computer vision and pattern recognition}}. \bibinfo{pages}{1110--1118}.
\newblock


\bibitem[Ermolov et~al\mbox{.}(2022)]%
        {ermolov2022hyperbolic}
\bibfield{author}{\bibinfo{person}{Aleksandr Ermolov}, \bibinfo{person}{Leyla Mirvakhabova}, \bibinfo{person}{Valentin Khrulkov}, \bibinfo{person}{Nicu Sebe}, {and} \bibinfo{person}{Ivan Oseledets}.} \bibinfo{year}{2022}\natexlab{}.
\newblock \showarticletitle{Hyperbolic vision transformers: Combining improvements in metric learning}. In \bibinfo{booktitle}{\emph{Proceedings of the IEEE/CVF Conference on Computer Vision and Pattern Recognition}}. \bibinfo{pages}{7409--7419}.
\newblock


\bibitem[Ganea et~al\mbox{.}(2018)]%
        {ganea2018hyperbolic}
\bibfield{author}{\bibinfo{person}{Octavian Ganea}, \bibinfo{person}{Gary B{\'e}cigneul}, {and} \bibinfo{person}{Thomas Hofmann}.} \bibinfo{year}{2018}\natexlab{}.
\newblock \showarticletitle{Hyperbolic neural networks}.
\newblock \bibinfo{journal}{\emph{Advances in neural information processing systems}}  \bibinfo{volume}{31} (\bibinfo{year}{2018}).
\newblock


\bibitem[Gu and Dao(2023)]%
        {gu2023mamba}
\bibfield{author}{\bibinfo{person}{Albert Gu} {and} \bibinfo{person}{Tri Dao}.} \bibinfo{year}{2023}\natexlab{}.
\newblock \showarticletitle{Mamba: Linear-time sequence modeling with selective state spaces}.
\newblock \bibinfo{journal}{\emph{arXiv preprint arXiv:2312.00752}} (\bibinfo{year}{2023}).
\newblock


\bibitem[Han et~al\mbox{.}(2023)]%
        {han2023flatten}
\bibfield{author}{\bibinfo{person}{Dongchen Han}, \bibinfo{person}{Xuran Pan}, \bibinfo{person}{Yizeng Han}, \bibinfo{person}{Shiji Song}, {and} \bibinfo{person}{Gao Huang}.} \bibinfo{year}{2023}\natexlab{}.
\newblock \showarticletitle{Flatten transformer: Vision transformer using focused linear attention}. In \bibinfo{booktitle}{\emph{Proceedings of the IEEE/CVF international conference on computer vision}}. \bibinfo{pages}{5961--5971}.
\newblock


\bibitem[Huang et~al\mbox{.}(2023)]%
        {huang2023hyperbolic}
\bibfield{author}{\bibinfo{person}{Wenkai Huang}, \bibinfo{person}{Yujia Yu}, \bibinfo{person}{Haizhou Xu}, \bibinfo{person}{Zhiwen Su}, {and} \bibinfo{person}{Yu Wu}.} \bibinfo{year}{2023}\natexlab{}.
\newblock \showarticletitle{Hyperbolic music transformer for structured music generation}.
\newblock \bibinfo{journal}{\emph{IEEE Access}}  \bibinfo{volume}{11} (\bibinfo{year}{2023}), \bibinfo{pages}{26893--26905}.
\newblock


\bibitem[Li et~al\mbox{.}(2017a)]%
        {li2017joint}
\bibfield{author}{\bibinfo{person}{Chuankun Li}, \bibinfo{person}{Yonghong Hou}, \bibinfo{person}{Pichao Wang}, {and} \bibinfo{person}{Wanqing Li}.} \bibinfo{year}{2017}\natexlab{a}.
\newblock \showarticletitle{Joint distance maps based action recognition with convolutional neural networks}.
\newblock \bibinfo{journal}{\emph{IEEE Signal Processing Letters}} \bibinfo{volume}{24}, \bibinfo{number}{5} (\bibinfo{year}{2017}), \bibinfo{pages}{624--628}.
\newblock


\bibitem[Li et~al\mbox{.}(2017b)]%
        {li2017skeleton}
\bibfield{author}{\bibinfo{person}{Chuankun Li}, \bibinfo{person}{Pichao Wang}, \bibinfo{person}{Shuang Wang}, \bibinfo{person}{Yonghong Hou}, {and} \bibinfo{person}{Wanqing Li}.} \bibinfo{year}{2017}\natexlab{b}.
\newblock \showarticletitle{Skeleton-based action recognition using LSTM and CNN}. In \bibinfo{booktitle}{\emph{2017 IEEE International conference on multimedia \& expo workshops (ICMEW)}}. IEEE, \bibinfo{pages}{585--590}.
\newblock


\bibitem[Li et~al\mbox{.}(2017c)]%
        {li2017adaptive}
\bibfield{author}{\bibinfo{person}{Wenbo Li}, \bibinfo{person}{Longyin Wen}, \bibinfo{person}{Ming-Ching Chang}, \bibinfo{person}{Ser Nam~Lim}, {and} \bibinfo{person}{Siwei Lyu}.} \bibinfo{year}{2017}\natexlab{c}.
\newblock \showarticletitle{Adaptive RNN tree for large-scale human action recognition}. In \bibinfo{booktitle}{\emph{Proceedings of the IEEE international conference on computer vision}}. \bibinfo{pages}{1444--1452}.
\newblock


\bibitem[Parmar et~al\mbox{.}(2018)]%
        {parmar2018image}
\bibfield{author}{\bibinfo{person}{Niki Parmar}, \bibinfo{person}{Ashish Vaswani}, \bibinfo{person}{Jakob Uszkoreit}, \bibinfo{person}{Lukasz Kaiser}, \bibinfo{person}{Noam Shazeer}, \bibinfo{person}{Alexander Ku}, {and} \bibinfo{person}{Dustin Tran}.} \bibinfo{year}{2018}\natexlab{}.
\newblock \showarticletitle{Image transformer}. In \bibinfo{booktitle}{\emph{International conference on machine learning}}. PMLR, \bibinfo{pages}{4055--4064}.
\newblock


\bibitem[Peng et~al\mbox{.}(2024)]%
        {peng2024eagle}
\bibfield{author}{\bibinfo{person}{Bo Peng}, \bibinfo{person}{Daniel Goldstein}, \bibinfo{person}{Quentin Anthony}, \bibinfo{person}{Alon Albalak}, \bibinfo{person}{Eric Alcaide}, \bibinfo{person}{Stella Biderman}, \bibinfo{person}{Eugene Cheah}, \bibinfo{person}{Xingjian Du}, \bibinfo{person}{Teddy Ferdinan}, \bibinfo{person}{Haowen Hou}, {et~al\mbox{.}}} \bibinfo{year}{2024}\natexlab{}.
\newblock \showarticletitle{Eagle and finch: Rwkv with matrix-valued states and dynamic recurrence}.
\newblock \bibinfo{journal}{\emph{arXiv preprint arXiv:2404.05892}} (\bibinfo{year}{2024}).
\newblock


\bibitem[Plizzari et~al\mbox{.}(2021a)]%
        {plizzari2021skeleton}
\bibfield{author}{\bibinfo{person}{Chiara Plizzari}, \bibinfo{person}{Marco Cannici}, {and} \bibinfo{person}{Matteo Matteucci}.} \bibinfo{year}{2021}\natexlab{a}.
\newblock \showarticletitle{Skeleton-based action recognition via spatial and temporal transformer networks}.
\newblock \bibinfo{journal}{\emph{Computer Vision and Image Understanding}}  \bibinfo{volume}{208} (\bibinfo{year}{2021}), \bibinfo{pages}{103219}.
\newblock


\bibitem[Plizzari et~al\mbox{.}(2021b)]%
        {plizzari2021spatial}
\bibfield{author}{\bibinfo{person}{Chiara Plizzari}, \bibinfo{person}{Marco Cannici}, {and} \bibinfo{person}{Matteo Matteucci}.} \bibinfo{year}{2021}\natexlab{b}.
\newblock \showarticletitle{Spatial temporal transformer network for skeleton-based action recognition}. In \bibinfo{booktitle}{\emph{Pattern recognition. ICPR international workshops and challenges: virtual event, January 10--15, 2021, Proceedings, Part III}}. Springer, \bibinfo{pages}{694--701}.
\newblock


\bibitem[Qiu et~al\mbox{.}(2022)]%
        {qiu2022spatio}
\bibfield{author}{\bibinfo{person}{Helei Qiu}, \bibinfo{person}{Biao Hou}, \bibinfo{person}{Bo Ren}, {and} \bibinfo{person}{Xiaohua Zhang}.} \bibinfo{year}{2022}\natexlab{}.
\newblock \showarticletitle{Spatio-temporal tuples transformer for skeleton-based action recognition}.
\newblock \bibinfo{journal}{\emph{arXiv preprint arXiv:2201.02849}} (\bibinfo{year}{2022}).
\newblock


\bibitem[Qu et~al\mbox{.}(2024)]%
        {qu2024llms}
\bibfield{author}{\bibinfo{person}{Haoxuan Qu}, \bibinfo{person}{Yujun Cai}, {and} \bibinfo{person}{Jun Liu}.} \bibinfo{year}{2024}\natexlab{}.
\newblock \showarticletitle{Llms are good action recognizers}. In \bibinfo{booktitle}{\emph{Proceedings of the IEEE/CVF Conference on Computer Vision and Pattern Recognition}}. \bibinfo{pages}{18395--18406}.
\newblock


\bibitem[Ren et~al\mbox{.}(2024)]%
        {ren2024survey}
\bibfield{author}{\bibinfo{person}{Bin Ren}, \bibinfo{person}{Mengyuan Liu}, \bibinfo{person}{Runwei Ding}, {and} \bibinfo{person}{Hong Liu}.} \bibinfo{year}{2024}\natexlab{}.
\newblock \showarticletitle{A survey on 3d skeleton-based action recognition using learning method}.
\newblock \bibinfo{journal}{\emph{Cyborg and Bionic Systems}}  \bibinfo{volume}{5} (\bibinfo{year}{2024}), \bibinfo{pages}{0100}.
\newblock


\bibitem[Tsai et~al\mbox{.}(2019)]%
        {tsai2019transformer}
\bibfield{author}{\bibinfo{person}{Yao-Hung~Hubert Tsai}, \bibinfo{person}{Shaojie Bai}, \bibinfo{person}{Makoto Yamada}, \bibinfo{person}{Louis-Philippe Morency}, {and} \bibinfo{person}{Ruslan Salakhutdinov}.} \bibinfo{year}{2019}\natexlab{}.
\newblock \showarticletitle{Transformer dissection: a unified understanding of transformer's attention via the lens of kernel}.
\newblock \bibinfo{journal}{\emph{arXiv preprint arXiv:1908.11775}} (\bibinfo{year}{2019}).
\newblock


\bibitem[Vaswani(2017)]%
        {vaswani2017attention}
\bibfield{author}{\bibinfo{person}{A Vaswani}.} \bibinfo{year}{2017}\natexlab{}.
\newblock \showarticletitle{Attention is all you need}.
\newblock \bibinfo{journal}{\emph{Advances in Neural Information Processing Systems}} (\bibinfo{year}{2017}).
\newblock


\bibitem[Wang et~al\mbox{.}(2018)]%
        {wang2018action}
\bibfield{author}{\bibinfo{person}{Pichao Wang}, \bibinfo{person}{Wanqing Li}, \bibinfo{person}{Chuankun Li}, {and} \bibinfo{person}{Yonghong Hou}.} \bibinfo{year}{2018}\natexlab{}.
\newblock \showarticletitle{Action recognition based on joint trajectory maps with convolutional neural networks}.
\newblock \bibinfo{journal}{\emph{Knowledge-Based Systems}}  \bibinfo{volume}{158} (\bibinfo{year}{2018}), \bibinfo{pages}{43--53}.
\newblock


\bibitem[Wu et~al\mbox{.}(2024)]%
        {wu2024frequency}
\bibfield{author}{\bibinfo{person}{Wenhan Wu}, \bibinfo{person}{Ce Zheng}, \bibinfo{person}{Zihao Yang}, \bibinfo{person}{Chen Chen}, \bibinfo{person}{Srijan Das}, {and} \bibinfo{person}{Aidong Lu}.} \bibinfo{year}{2024}\natexlab{}.
\newblock \showarticletitle{Frequency Guidance Matters: Skeletal Action Recognition by Frequency-Aware Mixed Transformer}. In \bibinfo{booktitle}{\emph{Proceedings of the 32nd ACM International Conference on Multimedia}}. \bibinfo{pages}{4660--4669}.
\newblock


\bibitem[Yan et~al\mbox{.}(2018)]%
        {yan2018spatial}
\bibfield{author}{\bibinfo{person}{Sijie Yan}, \bibinfo{person}{Yuanjun Xiong}, {and} \bibinfo{person}{Dahua Lin}.} \bibinfo{year}{2018}\natexlab{}.
\newblock \showarticletitle{Spatial temporal graph convolutional networks for skeleton-based action recognition}. In \bibinfo{booktitle}{\emph{Proceedings of the AAAI conference on artificial intelligence}}, Vol.~\bibinfo{volume}{32}.
\newblock


\bibitem[Yang et~al\mbox{.}(2024)]%
        {yang2024hypformer}
\bibfield{author}{\bibinfo{person}{Menglin Yang}, \bibinfo{person}{Harshit Verma}, \bibinfo{person}{Delvin~Ce Zhang}, \bibinfo{person}{Jiahong Liu}, \bibinfo{person}{Irwin King}, {and} \bibinfo{person}{Rex Ying}.} \bibinfo{year}{2024}\natexlab{}.
\newblock \showarticletitle{Hypformer: Exploring efficient transformer fully in hyperbolic space}. In \bibinfo{booktitle}{\emph{Proceedings of the 30th ACM SIGKDD Conference on Knowledge Discovery and Data Mining}}. \bibinfo{pages}{3770--3781}.
\newblock


\bibitem[Zhang and Gao(2021)]%
        {zhang2021hype}
\bibfield{author}{\bibinfo{person}{Chengkun Zhang} {and} \bibinfo{person}{Junbin Gao}.} \bibinfo{year}{2021}\natexlab{}.
\newblock \showarticletitle{Hype-han: Hyperbolic hierarchical attention network for semantic embedding}. In \bibinfo{booktitle}{\emph{Proceedings of the Twenty-Ninth International Conference on International Joint Conferences on Artificial Intelligence}}. \bibinfo{pages}{3990--3996}.
\newblock


\bibitem[Zhang et~al\mbox{.}(2022)]%
        {zhang2022zoom}
\bibfield{author}{\bibinfo{person}{Jiaxu Zhang}, \bibinfo{person}{Yifan Jia}, \bibinfo{person}{Wei Xie}, {and} \bibinfo{person}{Zhigang Tu}.} \bibinfo{year}{2022}\natexlab{}.
\newblock \showarticletitle{Zoom transformer for skeleton-based group activity recognition}.
\newblock \bibinfo{journal}{\emph{IEEE Transactions on Circuits and Systems for Video Technology}} \bibinfo{volume}{32}, \bibinfo{number}{12} (\bibinfo{year}{2022}), \bibinfo{pages}{8646--8659}.
\newblock


\bibitem[Zhang et~al\mbox{.}(2021)]%
        {zhang2021stst}
\bibfield{author}{\bibinfo{person}{Yuhan Zhang}, \bibinfo{person}{Bo Wu}, \bibinfo{person}{Wen Li}, \bibinfo{person}{Lixin Duan}, {and} \bibinfo{person}{Chuang Gan}.} \bibinfo{year}{2021}\natexlab{}.
\newblock \showarticletitle{STST: Spatial-temporal specialized transformer for skeleton-based action recognition}. In \bibinfo{booktitle}{\emph{Proceedings of the 29th ACM International Conference on Multimedia}}. \bibinfo{pages}{3229--3237}.
\newblock


\bibitem[Zhao et~al\mbox{.}(2021)]%
        {zhao2021point}
\bibfield{author}{\bibinfo{person}{Hengshuang Zhao}, \bibinfo{person}{Li Jiang}, \bibinfo{person}{Jiaya Jia}, \bibinfo{person}{Philip~HS Torr}, {and} \bibinfo{person}{Vladlen Koltun}.} \bibinfo{year}{2021}\natexlab{}.
\newblock \showarticletitle{Point transformer}. In \bibinfo{booktitle}{\emph{Proceedings of the IEEE/CVF international conference on computer vision}}. \bibinfo{pages}{16259--16268}.
\newblock


\bibitem[Zhou et~al\mbox{.}(2024)]%
        {zhou2024blockgcn}
\bibfield{author}{\bibinfo{person}{Yuxuan Zhou}, \bibinfo{person}{Xudong Yan}, \bibinfo{person}{Zhi-Qi Cheng}, \bibinfo{person}{Yan Yan}, \bibinfo{person}{Qi Dai}, {and} \bibinfo{person}{Xian-Sheng Hua}.} \bibinfo{year}{2024}\natexlab{}.
\newblock \showarticletitle{BlockGCN: Redefine Topology Awareness for Skeleton-Based Action Recognition}. In \bibinfo{booktitle}{\emph{Proceedings of the IEEE/CVF Conference on Computer Vision and Pattern Recognition}}. \bibinfo{pages}{2049--2058}.
\newblock


\end{thebibliography}
\end{document}